\documentclass{article}
\usepackage[accepted]{icml2019}
\usepackage{microtype,url}
\usepackage{graphicx}
\usepackage{subfigure}
\usepackage{booktabs} % for professional tables

\usepackage{ogonek}
\usepackage{soul}
\usepackage{xcolor}
\usepackage{multirow}
\usepackage{amsmath,amssymb,amsthm}
\newcommand{\Prox}{{\mathrm{Prox}}}
\newcommand{\argmin}{\mathop{\mathrm{argmin}}}  
\newcommand{\reals}{\mathbb{R}}
\newtheorem{theorem}{Theorem}
\newtheorem{lemma}[theorem]{Lemma}
\newtheorem{corollary}[theorem]{Corollary}
\usepackage{hyperref}

\icmltitlerunning{Plug-and-Play Methods Provably Converge with Properly Trained Denoisers}

\begin{document}

\twocolumn[
\icmltitle{Plug-and-Play Methods Provably Converge with Properly Trained Denoisers}




\begin{icmlauthorlist}
\icmlauthor{Ernest K. Ryu}{to}
\icmlauthor{Jialin Liu}{to}
\icmlauthor{Sicheng Wang}{goo}
\icmlauthor{Xiaohan Chen}{goo}
\icmlauthor{Zhangyang Wang}{goo}
\icmlauthor{Wotao Yin}{to}
\end{icmlauthorlist}

\icmlaffiliation{to}{Department of Mathematics, University of California, Los Angeles, USA}
\icmlaffiliation{goo}{Department of Computer Science and Engineering, Texas A\&M University, USA}

\icmlcorrespondingauthor{Wotao Yin}{wotaoyin\@math.ucla.edu}

\icmlkeywords{Plug and play, spectral normalization, Douglas-Rachford spiltting, ADMM}

\vskip 0.3in
]



\printAffiliationsAndNotice{}  % leave blank if no need to mention equal contribution

\begin{abstract}
Plug-and-play (PnP) is a non-convex framework that integrates modern denoising priors, such as BM3D or deep learning-based denoisers, into ADMM or other proximal algorithms. An advantage of PnP is that one  can use pre-trained denoisers when there is not sufficient data for end-to-end training. Although PnP has been recently studied extensively with great empirical success, theoretical analysis addressing even the most basic question of convergence has been insufficient. In this paper, we theoretically establish convergence of PnP-FBS and PnP-ADMM, without using diminishing stepsizes, under a certain Lipschitz condition on the denoisers. We then propose real spectral normalization, a technique for training deep learning-based denoisers to satisfy the proposed Lipschitz condition. Finally, we present experimental results validating the theory.
\end{abstract}

\section{Introduction}
\label{sec:intro}
Many modern image processing algorithms
recover or denoise an image through the optimization problem
\[
\begin{array}{ll}\label{first}
\underset{x\in \reals^d}{
\mbox{minimize}}&
 f(x)+\gamma g(x),
\end{array}
\]
where the optimization variable $x\in \reals^d$ represents the image,
$f(x)$ measures data fidelity,
$g(x)$ measures noisiness or complexity of the image,
and $\gamma\ge 0$ is a parameter representing the relative importance between $f$ and $g$.
Total variation denoising, inpainting, and compressed sensing fall under this setup.
\emph{A priori} knowledge of the image, such as that the image should have small noise, is encoded in $g(x)$.
So $g(x)$ is small if $x$ has small noise or complexity.
\emph{A posteriori} knowledge of the image, such as noisy or partial measurements of the image, is encoded in $f(x)$.
So $f(x)$ is small if $x$ agrees with the measurements.

First-order iterative methods are often used to solve such optimization problems,
and ADMM is one such method:
\begin{align*}
x^{k+1}&=
\argmin_{x\in \mathbb{R}^d}\left\{
\sigma^2 g(x)+(1/2)\|x-(y^{k}-u^{k})\|^2\right\}\nonumber\\
y^{k+1}&=\argmin_{y\in \mathbb{R}^d}\left\{
\alpha f(y)+(1/2)\|y-(x^{k+1}+u^k)\|^2
\right\}
%\tag{ADMM}
\\
u^{k+1}&=u^k+x^{k+1}-y^{k+1}\nonumber
\end{align*}
with $\sigma^2=\alpha\gamma $.
Given a function $h$ on $\reals^d$ and $\alpha>0$, define the proximal operator of $h$ as
\[
\Prox_{\alpha h}(z)=\argmin_{x\in \reals^d}\left\{
\alpha h(x)+(1/2)\|x-z\|^2
\right\},
\]
which is well-defined if $h$ is proper, closed, and convex.
Now we can equivalently write ADMM as
\begin{align*}
x^{k+1}&=
\Prox_{\sigma^2 g}(y^k-u^k)\nonumber\\
y^{k+1}&=
\Prox_{\alpha f}(x^{k+1}+u^k)
%\tag{ADMM}
\\
u^{k+1}&=u^k+x^{k+1}-y^{k+1}.\nonumber
\end{align*}
We can interpret the subroutine $\Prox_{\sigma^2 g}:\reals^d\rightarrow\reals^d$ as a denoiser, i.e.,
\[
\Prox_{\sigma^2 g}:\text{noisy image}\mapsto \text{less noisy image}
\]
(For example, if $\sigma$ is the noise level and $g(x)$ is the total variation (TV) norm, then $\Prox_{\sigma^2 g}$ is the standard Rudin--Osher--Fatemi (ROF) model \cite{rudin1992nonlinear}.)
We can think of $\Prox_{\alpha f}:\reals^d\rightarrow\reals^d$ 
as a mapping enforcing  consistency with measured data, i.e.,
\[
\Prox_{\alpha f}:\text{less consistent}\mapsto \text{more consistent with data}
\]
More precisely speaking, 
for any $x\in \reals^d$ we have
\[
g(\Prox_{\sigma^2 g}(x))\le g(x),
\qquad
f(\Prox_{\alpha f}(x))\le f(x).
\]

However, some state-of-the-art image denoisers with great empirical performance
do not originate from optimization problems.
Such examples include non-local means (NLM)
\cite{buades2005},
Block-matching and 3D filtering
(BM3D) \cite{dabov2007}, and convolutional neural networks (CNN) \cite{zhang2017beyond}. 
Nevertheless, such a denoiser $H_\sigma:\reals^d\rightarrow\reals^d$ still has the interpretation
\[
H_\sigma:\text{noisy image}\mapsto \text{less noisy image}
\]
where $\sigma\ge 0$ is a noise parameter.
Larger values of $\sigma$ correspond to more aggressive denoising.

Is it possible to use such denoisers for a broader range of imaging
problems,
even though we cannot directly set up an optimization problem?
To address this question, \cite{venkatkrishnan2013}
proposed Plug-and-Play ADMM (PnP-ADMM), which simply replaces the proximal operator $\Prox_{\sigma^2 g}$ with the denoiser $H_\sigma$:
\begin{align*}
x^{k+1}&=H_\sigma(y^{k}-u^{k})\nonumber\\
%y^{k+1}&=\argmin_{y\in \mathbb{R}^d}\left\{f(y)+\frac{1}{2\alpha}\|y-x^{k+1}-u^k\|^2\right\}\tag{PNP-ADMM}\\
y^{k+1}&=\Prox_{\alpha f}(x^{k+1}+u^k)\\
u^{k+1}&=u^k+x^{k+1}-y^{k+1}.\nonumber
\end{align*}
Surprisingly and remarkably, this ad-hoc method exhibited great empirical success, and spurred much follow-up work.

\vspace{-1em}
\paragraph{Contribution of this paper.}
The empirical success of Plug-and-Play (PnP) naturally leads us to ask theoretical questions:
When does PnP converge and what denoisers can we use?
Past theoretical analysis has been insufficient. 
%When does PnP converge to?

The main contribution of this work is the convergence analyses of PnP methods. We study two Plug-and-play methods, Plug-and-play forward-backward splitting (PNP-FBS) and PNP-ADMM. 
For the analysis, we assume the denoiser $H_\sigma$ satisfies a certain Lipschitz condition, formally defined as Assumption~\eqref{assumption:H}.
%we assume $f$ is strongly convex and
Roughly speaking, the condition corresponds to the denoiser $H_\sigma$ being close to the identity map, which is reasonable when the denoising parameter $\sigma$ is small.
In particular, we do not assume that $H_\sigma$ is nonexpansive or differentiable since most denoisers do not have such properties.
Under the assumption, we show that the PnP methods are contractive iterations.

We then propose real spectral normalization (realSN), a technique based on \cite{miyato2018spectral} for more accurately constraining deep learning-based denoisers in their training to satisfy the proposed Lipschitz condition.
%in Section~\ref{s:denoiser}.
Finally, we present experimental results validating our theory. %Sections \ref{s:poisson_experiments} and \ref{s:more_apple}.
Code used for experiments is available at:
\url{https://github.com/uclaopt/Provable_Plug_and_Play/}

\subsection{Prior work}
\paragraph{Plug-and-play: Practice.}
The first PnP method was the Plug-and-play ADMM proposed in \cite{venkatkrishnan2013}. Since then, other schemes such as
the primal-dual method \cite{heide2014flexisp,Meinhardt2017,ono_2017},
ADMM with increasing penalty parameter \cite{brifman2016,chan2017},
generalized approximate message passing \cite{metzler2016},
Newton iteration \cite{buzzard2017},
Fast Iterative Shrinkage-Thresholding Algorithm \cite{Kamilov2017,sun_pnp_2018},
(stochastic) forward-backward splitting \cite{sun_image_pnp_2018,sun_pnp_sgd_2018,sun_pnp_2018}, and alternating minimization \cite{DongWangYinShi2018_denoising}
have been combined with the PnP technique.

PnP method reported empirical success
on a large variety of imaging applications:
bright field electron tomography \cite{sreehari2016},
camera image processing \cite{heide2014flexisp},
compression-artifact reduction \cite{dar2016},
compressive imaging \cite{teodoro2016},
deblurring \cite{teodoro2016,rond2016,wang2017},
electron microscopy \cite{sreehari2017},
Gaussian denoising \cite{buzzard2017,DongWangYinShi2018_denoising},
nonlinear inverse scattering \cite{Kamilov2017},
Poisson denoising \cite{rond2016},
single-photon imaging \cite{chan2017},
super-resolution \cite{brifman2016,sreehari2016,chan2017},
diffraction tomography \cite{sun_image_pnp_2018},
Fourier ptychographic microscopy \cite{sun_pnp_2018},
low-dose CT imaging \cite{venkatkrishnan2013,he2018,ye2018,lyu_pnp_2019},
hyperspectral sharpening \cite{teodoro2017,teodoro_2019},
inpainting \cite{chan_pnp_2019,tirer2019image}, and  superresolution \cite{DongWangYinShi2018_denoising}.

%\cite{venkatkrishnan2013,heide2014flexisp,brifman2016,dar2016,rond2016,sreehari2016,teodoro2016,buzzard2017, chan2017,Kamilov2017,sreehari2017,wang2017}.

A wide range of denoisers have been used for the PnP framework.
BM3D has been used the most \cite{heide2014flexisp,dar2016,rond2016,sreehari2016,chan2017,Kamilov2017,ono_2017,wang2017}, but other denoisers such as sparse representation \cite{brifman2016},
non-local means \cite{venkatkrishnan2013,heide2014flexisp,sreehari2016,sreehari2017,chan_pnp_2019},
Gaussian mixture model \cite{teodoro2016,teodoro2017,shi2018,teodoro_2019},
Patch-based Wiener filtering \cite{venkatkrishnan2013},
nuclear norm minimization
\cite{Kamilov2017},
deep learning-based denoisers \cite{Meinhardt2017,he2018,ye2018,tirer2019image}
and deep projection model based on generative adversarial networks \cite{chang2017one}  have also been considered.

\vspace{-1em}
\paragraph{Plug-and-play: Theory.}
Compared to the empirical success, much less progress was made on the theoretical aspects of PnP optimization.
\cite{chan2017} analyzed convergence with a bounded denoiser assumption, establishing convergence using an increasing penalty parameter.
\cite{buzzard2017} provided an interpretation of fixed points via ``consensus equilibrium''.
\cite{sreehari2016,sun_image_pnp_2018,teodoro2017,chan_pnp_2019,teodoro_2019} proved convergence of PNP-ADMM and PNP-FBS with the assumption that the denoiser is (averaged) nonexpansive by viewing the methods to be fixed-point iterations. The nonexpansiveness assumption is not met with most denoisers as is, but \cite{chan_pnp_2019} proposed modifications to the non-local means and Gaussian mixture model denoisers, which make them into linear filters, to enforce nonexpansiveness.   \cite{DongWangYinShi2018_denoising} presented a proof that relies on the existence of a certain Lyapunov function that is monotonic under $H_\sigma$, which holds only for simple $H_\sigma$.
\cite{tirer2019image} analyzed a variant of PnP, but did not establish local convergence since their key assumption is only expected to be satisfied ``in early iterations''.

\vspace{-1em}
\paragraph{Other PnP-type methods.}
There are other lines of works that incorporate modern denoisers into model-based optimization methods.
The plug-in idea with half quadratic splitting, as opposed to ADMM, was discussed \cite{zoran_2011} and this approach was carried out with deep learning-based denoisers in \cite{zhang_ircnn_2017}.
\cite{Danielyan2012,Egiazarian2015} use the notion of Nash equilibrium to propose a scheme similar to PnP.
\cite{Danielyan2010} proposed an augmented Lagrangian method similar to PnP.
\cite{romano2017,Reehorst2019} presented Regularization by Denoising (RED), which uses the (nonconvex) regularizer $x^T(x-H_\sigma(x))$ given a denoiser $H_\sigma$, and use denoiser evaluations in its iterations. 
\cite{fletcher2018} applies the plug-in approach to vector approximate message passing.
\cite{sun2016deep,fan2017inversenet} replaced both the proximal operator enforcing data fidelity and the denoiser with two neural networks and performed end-to-end training. Broadly, there are more works that incorporate model-based optimization with deep learning \cite{chen2018theoretical,liu2018alista}.

\vspace{-1em}
\paragraph{Image denoising using deep learning.}
Deep learning-based denoising methods have become state-of-the-art. \cite{zhang2017beyond} proposed an effective denoising network called DnCNN, which adopted batch normalization \cite{ioffe2015batch} and ReLU \cite{krizhevsky2012imagenet} into the residual learning \cite{he2016deep}. Other represenative deep denoising models include the deep convolutional encoder-decoder with symmetric skip connection \cite{mao2016image}, $N^3$Net \cite{plotz2018neural}, and MWCNN \cite{liu2018multi}. The recent FFDNet \cite{zhang2018ffdnet} handles spatially varying Gaussian noise.

\vspace{-1em}
\paragraph{Regularizing Lipschitz continuity.}
Lipschitz continuity and its variants have started to receive attention as a means for regularizing deep classifiers \cite{bartlett2017spectrally,bansal2018can,oberman2018lipschitz} and GANs \cite{miyato2018spectral,brock2018large}. Regularizing Lipschitz continuity stabilizes training, improves the final performance, and  enhances robustness to adversarial attacks \cite{weng2018evaluating,qian2018lnonexpansive}. Specifically, \cite{miyato2018spectral} proposed to normalize all weights to be of unit spectral norms to thereby constrain the Lipschitz constant of the overall network to be no more than one.

\section{PNP-FBS/ADMM and their fixed points}
We now present the PnP methods we investigate in this work.
We quickly note that although PNP-FBS and PNP-ADMM are distinct methods, 
they share the same fixed points by Remark 3.1 of \cite{Meinhardt2017} and Proposition 3 of \cite{sun_image_pnp_2018}.

%\paragraph{Plug-and-play forward-backward splitting}
We call the method 
\begin{align}
%x^{k+1}&= H_\sigma (I-\alpha \nabla f)(x^{k+1/2})
x^{k+1}&= H_\sigma (I-\alpha \nabla f)(x^{k})
\tag{PNP-FBS}
\end{align}
for any $\alpha>0$, plug-and-play forward-backward splitting (PNP-FBS) or plug-and-play proximal gradient method.
%The name comes from calling gradient step the ``forward step'' and the proximal operator the ``backward step'' in analogy to the forward and backward Euler steps, respectively.
% Forward-backward splitting is also known as the proximal gradient method, when $H_\sigma$ originates from a proximal operator.
%Forward-backward splitting without plug-and-play was first presented in \cite{passty1979}.

We interpret PNP-FBS as a fixed-point iteration,
and we say $x^\star$ is a fixed point of PNP-FBS if
\begin{align*}
x^{\star}&=H_\sigma(I-\alpha\nabla f)(x^\star).
\end{align*}
Fixed points of PNP-FBS have a simple, albeit non-rigorous, interpretation.
An image denoising algorithm must trade off the two goals of
making the image agree with measurements and
making the image less noisy.
PNP-FBS applies
$I-\alpha \nabla f$ and $H_\sigma$, each promoting such objectives, repeatedly in an alternating fashion.
If PNP-FBS converges to a fixed point, we can expect the limit to represent a compromise.

%\paragraph{Plug-and-play alternating directions method of multipliers}
We call the method
\begin{align}
x^{k+1}&=H_\sigma(y^{k}-u^{k})\nonumber\\
%y^{k+1}&=\argmin_{y\in \mathbb{R}^d}\left\{f(y)+\frac{1}{2\alpha}\|y-x^{k+1}-u^k\|^2\right\}\tag{PNP-ADMM}\\
y^{k+1}&=\Prox_{\alpha f}(x^{k+1}+u^k)
\tag{PNP-ADMM}\\
u^{k+1}&=u^k+x^{k+1}-y^{k+1}\nonumber
\end{align}
for any $\alpha>0$,
plug-and-play alternating directions method of multipliers (PNP-ADMM).
%ADMM without plug-and-play was first presented in \cite{gabay1976,glowinski1975}.
We interpret PNP-ADMM as a fixed-point iteration,
and we say $(x^\star,u^\star)$ is a fixed point of PNP-ADMM if
\begin{align*}
x^{\star}&=H_\sigma(x^{\star}-u^{\star})\\
%x^{\star}&=\argmin_{y\in \mathbb{R}^d}\left\{ f(y)+\frac{1}{2\alpha}\|y-x^{\star}-u^\star\|^2\right\}.
x^{\star}&=\Prox_{\alpha f}(x^{\star}+u^\star).
\end{align*}
If we let $y^k=x^\star$ and $u^k=u^\star$ in (PNP-ADMM), then we get $x^{k+1}=y^{k+1}=x^\star$ and $u^{k+1}=u^k=u^\star$.
We call the method
\begin{align}
x^{k+1/2}&=\Prox_{\alpha f}(z^k)\nonumber\\
x^{k+1}&=H_\sigma(2x^{k+1/2}-z^k)
\tag{PNP-DRS}
\\
z^{k+1}&=z^k+x^{k+1}-x^{k+1/2}\nonumber
\end{align}
plug-and-play Douglas--Rachford splitting (PNP-DRS).
We interpret PNP-DRS as a fixed-point iteration,
and we say $z^\star$ is a fixed point of PNP-DRS if
\begin{align*}
x^\star&=\Prox_{\alpha f}(z^\star)\\
x^\star&=H_\sigma(2x^\star-z^\star).
\end{align*}
%It is well known in the field of convex optimization that ADMM and DRS are equivalent. 
PNP-ADMM and PNP-DRS are equivalent.
Although this is not surprising as the equivalence between convex ADMM and DRS is well known, we show the steps establishing equivalence in the supplementary document.

We introduce PNP-DRS as an analytical tool for analyzing PNP-ADMM.
It is straightforward to verify that PNP-DRS can be written as $z^{k+1}=T(z^k)$, where
\[
T=\frac{1}{2}I+\frac{1}{2}(2H_\sigma-I)(2\Prox_{\alpha f}-I).
\]
%In the analysis of Theorem~\ref{thm:contraction},
We use this form to analyze the  convergence of PNP-DRS and translate the result to PNP-ADMM.

\section{Convergence via contraction}
\label{s:ct}
%\subsection{Assumption on $H_\sigma$ and its consequences}
%\label{ss:hassump}
We now present conditions that ensure the PnP methods are contractive and thereby convergent.
%and thereby prove that the methods converge.

If we assume $2H_\sigma-I$ is nonexpansive,
standard tools of monotone operator theory
%\cite{krasnoselskii1955,mann1953}
tell us that PnP-ADMM converges.
However, this assumption is too strong.
Chan et al.\ presented a counter example demonstrating that $2H_\sigma-I$ is not nonexpansive for the NLM denoiser \cite{chan2017}.

Rather, we assume $H_\sigma:\reals^d\rightarrow\reals^d$ satisfies
\begin{equation}
\|(H_\sigma-I)(x)-(H_\sigma-I)(y)\|^2\le \varepsilon^2\|x-y\|^2
\tag{A}
\label{assumption:H}
\end{equation}
for all $x,y\in \reals^d$ for some $\varepsilon\ge 0$.
Since $\sigma$ controls the strength of the denoising,
we can expect $H_\sigma$ to be close to identity for small $\sigma$.
If so , Assumption~\eqref{assumption:H} is reasonable.

Under this assumption, we show that the PNP-FBS and PNP-DRS iterations are \textbf{contractive}
in the sense that we can express the iterations as $x^{k+1}=T(x^k)$, where $T:\reals^d\rightarrow\reals^d$ satisfies 
\[
\|T(x)-T(y)\|\le \delta \|x-y\|
\]
for all $x,y\in \reals^d$ for some $\delta<1$.
We call $\delta$ the contraction factor.
If $x^\star$ satisfies $T(x^\star)=x^\star$, i.e., $x^\star$ is a fixed point, then $x^k\rightarrow x^\star$ geometrically
by the classical Banach contraction principle.%\cite{banach1922}.

%\subsection{Convergence results}
%\label{ss:PPfbs-conv}
\begin{theorem}[Convergence of PNP-FBS]
\label{thm:fbs-contraction}
Assume $H_\sigma$ satisfies assumption \eqref{assumption:H} for some $\varepsilon\ge 0$.
Assume $f$ is $\mu$-strongly convex, $f$ is differentiable, and $\nabla f$ is $L$-Lipschitz.
Then
\[
T = H_\sigma (I-\alpha \nabla f)
\]
satisfies
\[
\|T(x)-T(y)\|\le 
\max\{|1-\alpha\mu|,|1-\alpha L|\}(1+\varepsilon)
\|x-y\|
\]
for all $x,y\in \reals^d$.
The coefficient is less than $1$ if
\[
\frac{1}{\mu(1+1/\varepsilon)}<\alpha<
\frac{2}{L}-\frac{1}{L(1+1/\varepsilon)}.
\]
Such an $\alpha $ exists if $\varepsilon<2\mu/(L-\mu)$.
\end{theorem}

\newpage

\begin{theorem}[Convergence of PNP-DRS]
\label{thm:contraction}
Assume $H_\sigma$ satisfies assumption \eqref{assumption:H} for some $\varepsilon\ge 0$.
Assume $f$ is $\mu$-strongly convex and differentiable.
Then
\[
T=
\frac{1}{2}I+\frac{1}{2}
(2H_\sigma-I) (2\Prox_{\alpha f}-I)
\]
satisfies
\[
\|T(x)-T(y)\| \le
\frac{1+\varepsilon+\varepsilon\alpha\mu+2 \varepsilon^2\alpha\mu}{1+\alpha\mu+2\varepsilon\alpha\mu}
\|x-y\|
\]
for all $x,y\in \reals^d$.
The coefficient is less than $1$ if
\[
 \frac{\varepsilon}{(1+\varepsilon-2\varepsilon^2)\mu}< \alpha,\quad
\varepsilon<1.
\]
\end{theorem}

\begin{corollary}[Convergence of PNP-ADMM]
\label{cor:contraction}
Assume $H_\sigma$ satisfies assumption \eqref{assumption:H} for some $\varepsilon\in [0,1)$.
Assume $f$ is $\mu$-strongly convex.
Then PNP-ADMM converges for 
\vspace{-0.5em}
\[
 \frac{\varepsilon}{(1+\varepsilon-2\varepsilon^2)\mu}< \alpha.
\]
\end{corollary}
\begin{proof}
\vspace{-1em}
This follows from Theorem~\ref{thm:contraction}
and the equivalence of PNP-DRS and PNP-ADMM.
\end{proof}

%\subsection{Discussion}
%\label{ss:discuss}
For PNP-FBS, we assume
$f$ is $\mu$\nobreakdash-strongly convex and  $\nabla f$ is $L$-Lipschitz.
For PNP-DRS and PNP-ADMM, we assume
$f$ is $\mu$\nobreakdash-strongly convex.
These are standard assumptions
that are satisfied in application such as
% image denoising/deblurring \cite{rond2016,chan2017,wang2017} and single photon imaging \cite{chan2017}.
% The assumption of strong convexity, however, excludes a few applications 
% such as compressed sensing \cite{venkatkrishnan2013},
% sparse interpolation \cite{sreehari2016}, and super-resolution \cite{chan2017}.
image denoising/deblurring and single photon imaging.
Strong convexity, however, does exclude a few applications such as compressed sensing, sparse interpolation, and super-resolution.

PNP-FBS and PNP-ADMM are distinct methods for finding the same set of fixed points.
Sometimes, PNP-FBS is easier to implement since it only requires the computation of $\nabla f$ rather than $\Prox_{\alpha f}$.
On the other hand, PNP-ADMM has better convergence properties as demonstrated theoretically by
Theorems~\ref{thm:fbs-contraction} and \ref{thm:contraction}
and empirically by our experiments.
%of Section~\ref{s:poisson_experiments}.

The proof of Theorem~\ref{thm:contraction} relies on the notion of ``negatively averaged'' operators of \cite{giselsson2017}.
It is straightforward to modify Theorems~\ref{thm:fbs-contraction} and \ref{thm:contraction} to establish local convergence when Assumption~\eqref{assumption:H} holds locally.
Theorem~\ref{thm:contraction} can be generalized to the case when $f$ is strongly convex but non-differentiable using the notion of subgradients.
%We avoided this issue to avoid the discussion and notation of subgradients.

Recently, \cite{fletcher2018} proved convergence of ``plug-and-play'' vector approximate message passing, a method similar to ADMM, assuming Lipschitz continuity of the denoiser.
Although the method, the proof technique, and the notion of convergence are different from ours, the similarities are noteworthy.

\section{Real spectral normalization: 
enforcing Assumption~\eqref{assumption:H}}
\label{s:denoiser}
We now present real spectral normalization, a technique for training denoisers to satisfy Assumption~\eqref{assumption:H}
and connect the practical implementations to the theory of Section~\ref{s:ct}.

%we discuss the deep denoisers we use for the PnP method
\subsection{Deep learning denoisers: SimpleCNN and DnCNN}
We use a deep denoising model called \textbf{DnCNN} \cite{zhang2017beyond}, which learns the residual mapping  with a 17-layer CNN and reports state-of-the-art results on natural image denoising.
Given a noisy observation $y = x + e$, where $x$ is the clean image and $e$ is noise, the residual mapping $R$ outputs the noise, i.e., $R(y) = e$ so that $y - R(y)$ is the clean recovery.
Learning the residual mapping is a popular approach in deep learning-based image restoration.

We also construct a simple convolutional encoder-decoder model for denoising and call it \textbf{SimpleCNN}.
SimpleCNN consists of 4 convolutional layers, with ReLU and mean-square-error (MSE) loss and does not utilize any pooling or (batch) normalization. 

We remark that realSN and the theory of this work is applicable to other deep denoisers.
We use SimpleCNN to show that realSN is applicable to any CNN denoiser.

\subsection{Lipschitz constrained deep denoising}
\label{subsec:lip-DnCNN}
%We then propose the Lipschitz constrained CNN to align the practical use of deep denoisers with the theory of PnP framework.
Denote the denoiser (SimpleCNN or DnCNN) as $H(y)=y-R(y)$, where $y$ is the noisy input and $R$ is the residual mapping, i.e., $R(y)=y-H(y)=(I-H)(y)$. 
Enforcing Assumption~\eqref{assumption:H}
is equivalent to constraining the Lipschitz constant of $R(y)$.
We propose a variant of the spectral normalization (\textbf{SN}) \cite{miyato2018spectral} for this.

%\vspace{-1em}
\paragraph{Spectral normalization.} \cite{miyato2018spectral} proposed to normalize the spectral norm of each layer-wise weight (with ReLU non-linearity) to one. Provided that we use $1$-Lipschitz nonlinearities (such as ReLU), the Lipschitz constant of a layer is upper-bounded by the spectral norm of its weight, and the Lipschitz constant of the full network is bounded by the product of spectral norms of all layers \cite{gouk2018regularisation}.
%, which should have also been one after applying SN.
To avoid the prohibitive cost of singular value decomposition (SVD) every iteration, SN approximately computes the largest singular values of weights using a small number of power iterations.  

Given the weight matrix $W_l\in\mathbb{R}^{m\times n}$ of the $l$-th layer, vectors $u_l\in\mathbb{R}^m, v_l\in\mathbb{R}^m$ are initialized randomly and maintained in the memory to estimate the leading first left and right singular vector of $W_l$ respectively.
During each forward pass of the network, SN is applied to all layers $1\leq l \leq L$  following the two-step routine:
\vspace{-5pt}
\begin{enumerate} \itemsep -10pt
    \item Apply one step of the power method to update $u_l,v_l$:
        \begin{align*}
            &v_l \leftarrow W_l^T u_l\ /\ \|W_l^T u_l\|_2, ~~~~ u_l \leftarrow W_l v_l\ /\ \|W_l v_l\|_2
        \end{align*}
    \item Normalize $W_l$ with the estimated spectral norm:
        \[
            % \tilde{W}_\mathrm{SN}(W_l)
            W_l \leftarrow W_l / \sigma(W_l),
            \ \mathrm{where}\ \sigma(W_l) = u_l^T W_l v_l
        \]
\end{enumerate}

\vspace{-5pt}

While the basic methodology of SN suits our goal, the SN in \cite{miyato2018spectral} uses a convenient but inexact implementation for convolutional layers. A convolutional layer is represented by a four-dimensional kernel $K_l$ of shape $(C_\mathrm{out}, C_\mathrm{in}, h, w)$, where $h, w$ are kernel's height and width. SN reshapes $K_l$ into a two-dimensional matrix $\tilde{K}_l$ of shape $(C_\mathrm{out}, C_\mathrm{in} \times h \times w)$ and regards $\tilde{K}_l$ as the matrix $W_l$ above.
This relaxation \textbf{underestimates} the true spectral norm of the convolutional operator (Corollary 1 of \cite{tsuzuku2018lipschitz}) given by
\vspace{-2mm}
\[ \sigma(K_l) = \max_{x\neq 0} \|K_l * x\|_2 / \|x\|_2, 
\vspace{-2mm}\]
where $x$ is the input to the convolutional layer and $*$ is the convolutional operator.
This issue is not hypothetical.
When we trained SimpleCNN with SN, the spectral norms of the layers were $3.01$, $2.96$, $2.82$, and $1.31$, i.e., SN failed to control the Lipschitz constant below $1$.

%\vspace{-1em}
\paragraph{Real spectral normalization.}
We propose an improvement to SN for convolutional\footnote{We use stride 1 and  zero-pad with width 1 for convolutions.} layers, called the \textbf{real spectral normalization} (realSN), 
to more accurately constrain the network's Lipschitz constant
and thereby enforce Assumption~\eqref{assumption:H}.

In realSN, we directly consider the convolutional linear operator $\mathcal{K}_l:  \mathbb{R}^{C_\mathrm{in} \times h \times w} \to  \mathbb{R}^{C_\mathrm{out} \times h \times w}$, where
    $h,w$ are input's height and width,
    instead of reshaping the convolution kernel $K_l$ into a matrix.
    The power iteration also requies the conjugate (transpose) operator $\mathcal{K}^*_l$. It can be shown that $\mathcal{K}^*_l$ is another convolutional operator with a kernel that is a rotated version of the forward convolutional kernel; the first two dimensions are permuted and the last two dimensions are rotated by 180 degrees \citep{liu2018alista}.
    Instead of two vectors $u_l,v_l$ as in SN, realSN maintains $U_l\in \mathbb{R}^{C_\mathrm{out} \times h \times w}$ and 
    $V_l\in \mathbb{R}^{C_\mathrm{in} \times h \times w}$ to estimate the leading left and right singular vectors respectively. During each forward pass of the neural network, realSN conducts:
% and matrix-vector multiplications with convolutions:
\vspace{-5pt}
\begin{enumerate} \itemsep -10pt
    \item Apply one step of the power method with operator $\mathcal{K}_l$:
        \begin{align*}
            &V_l \leftarrow \mathcal{K}_l^*(U_l)\ /\ \|\mathcal{K}_l^*(U_l)\|_2, \\
            &U_l \leftarrow \mathcal{K}_l(V_l)\ /\ \|\mathcal{K}_l(V_l)\|_2.
        \end{align*}
    \item Normalize the convolutional kernel $K_l$ with estimated spectral norm:
        \[
            K_l \leftarrow K_l / \sigma(\mathcal{K}_l),
            \ \mathrm{where}\ \sigma(\mathcal{K}_l) = \langle U_l, \mathcal{K}_l( V_l ) \rangle
        \]
\end{enumerate}

\vspace{-5pt}
By replacing $\sigma(\mathcal{K}_l)$ with $\sigma(K_l)/c_l$, realSN can constrain the Lipschitz constant to any upper bound $C=\prod_{l=1}^L c_l$.
Using the highly efficient convolution computation in modern deep learning frameworks, realSN can be implemented simply and efficiently.
Specifically, realSN introduces three additional one-sample convolution operations for each layer in each training step. 
When we used a batch size of $128$, the extra computational cost of the additional operations is mild.

\begin{figure*}
\centering
\begin{tabular}{ccccc}
\hspace{-7mm}
\subfigure[BM3D]{
      \includegraphics[width=0.22\textwidth]{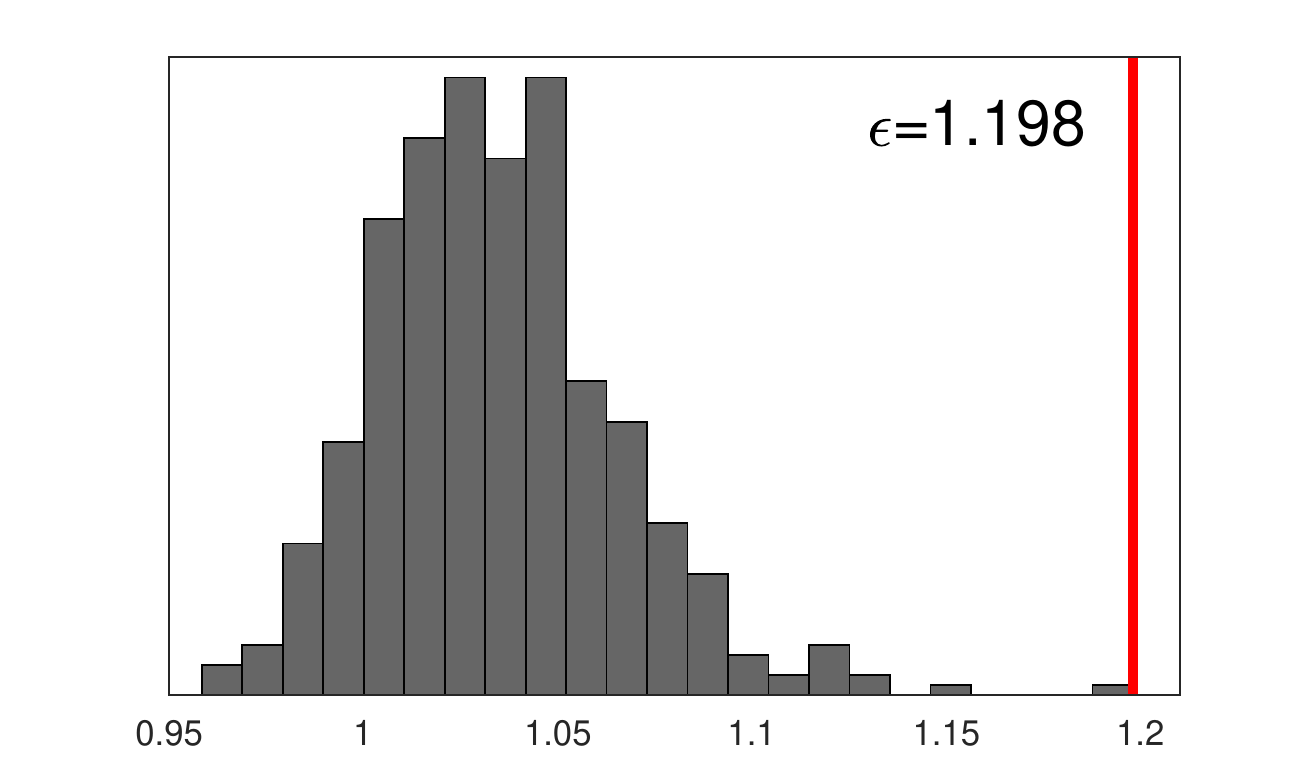}\label{fig:hist_bm3d}}
&
\hspace{-9mm}
\subfigure[SimpleCNN]{
      \includegraphics[width=0.22\textwidth]{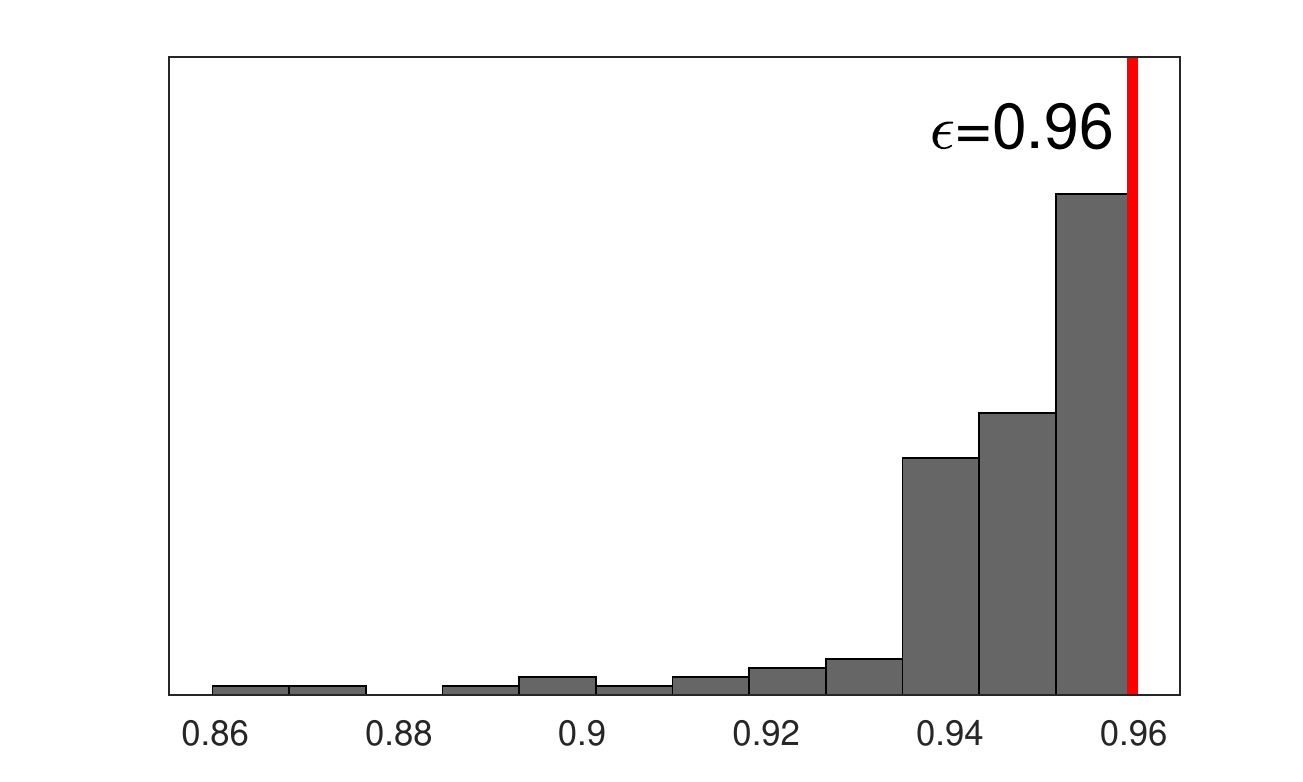}\label{fig:hist_simplecnn}}
  &
  \hspace{-9mm}
\subfigure[RealSN-SimpleCNN]{
      \includegraphics[width=0.22\textwidth]{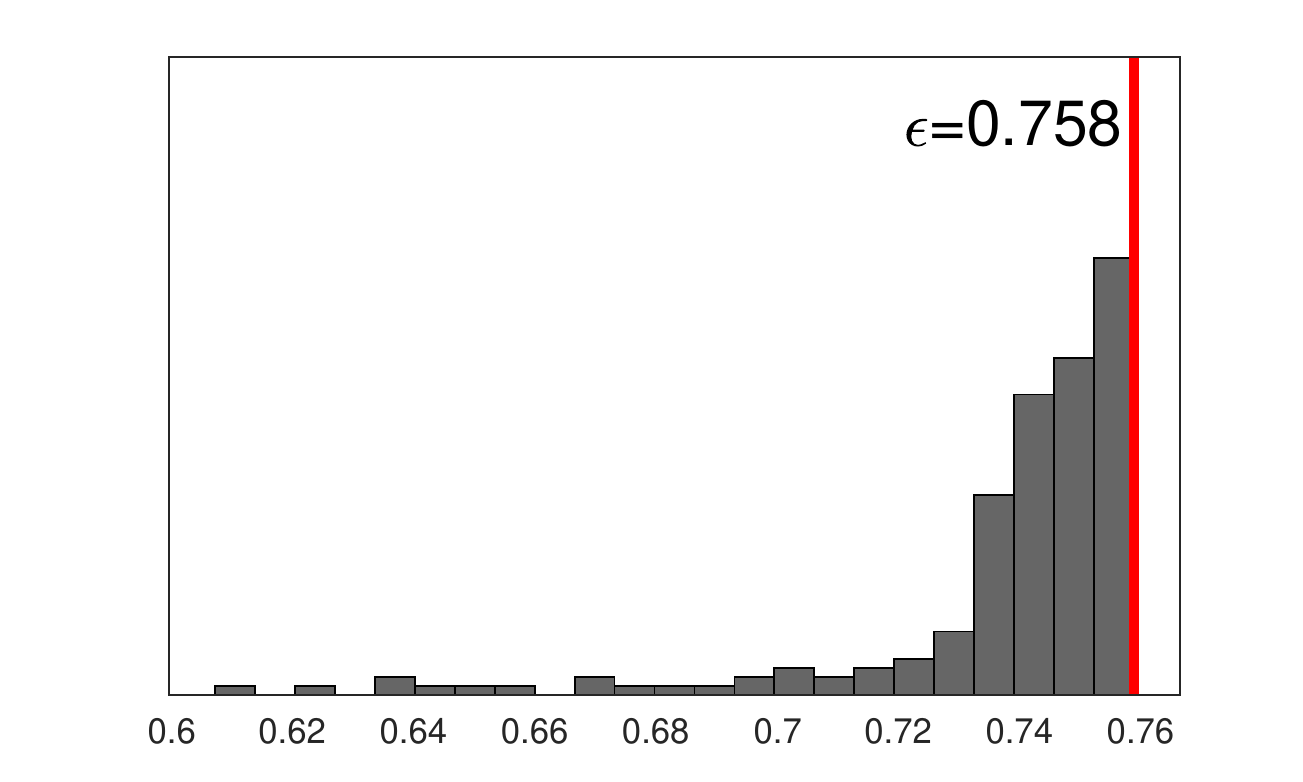}\label{fig:hist_simplecnn2}}
      &
      \hspace{-9mm}
\subfigure[DnCNN]{
      \includegraphics[width=0.22\textwidth]{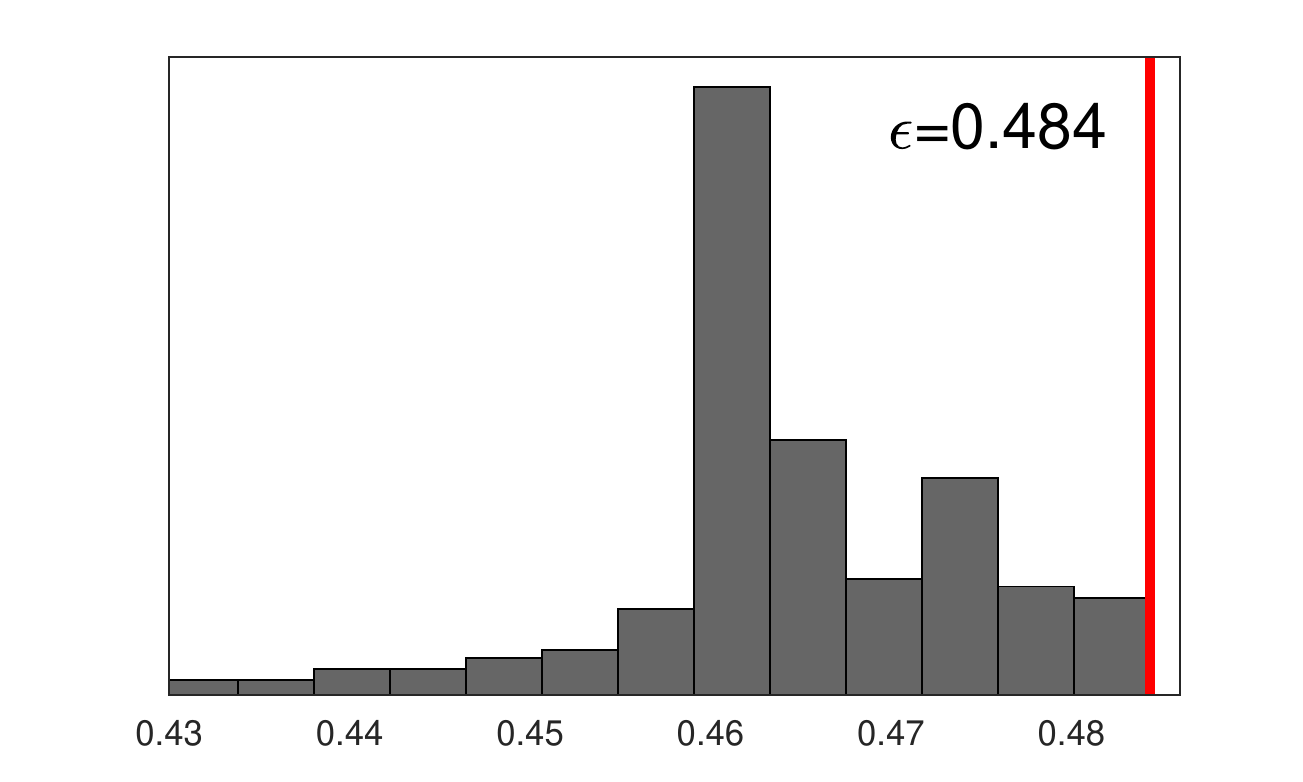}\label{fig:hist_dncnn}}
  &
\hspace{-9mm}
\subfigure[RealSN-DnCNN]{
      \includegraphics[width=0.22\textwidth]{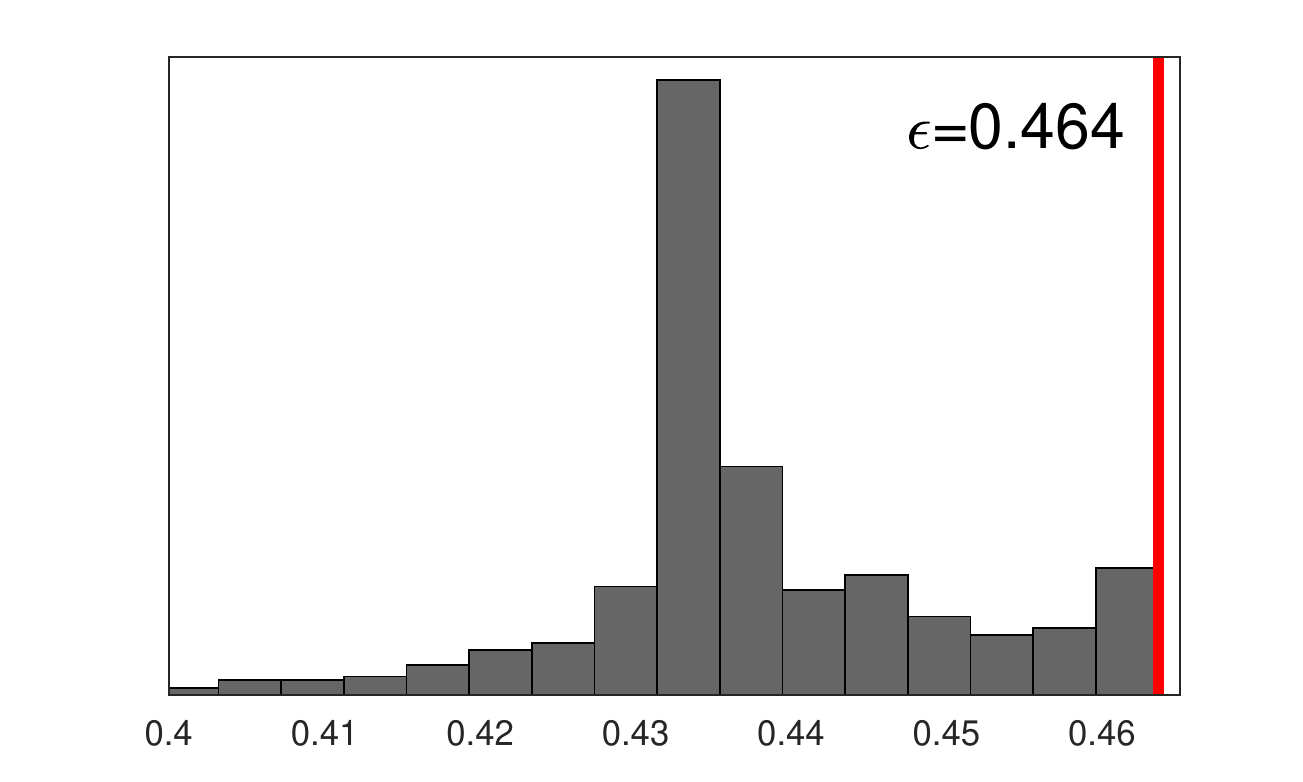}\label{fig:hist_dncnn2}}
\end{tabular}
\vspace{-0.5em}
\caption{Histograms for experimentally verifying Assumption~\eqref{assumption:H}.
The x-axis represents values of $\|(I-H_\sigma)(x)-(I-H_\sigma)(y)\|/\|x-y\|$ and the y-axis represents the frequency.
The vertical red bar corresponds to the maximum value.
}\vspace{-0.5em}
\label{fig:histogram}
\end{figure*}

\subsection{Implementation details}
%\subsubsection{Training and Testing Data}
We refer to SimpleCNN and DnCNN regularized by realSN as \textbf{RealSN-SimpleCNN} and \textbf{RealSN-DnCNN}, respectively. We train them in the setting of Gaussian denoising with known fixed noise levels $\sigma = 5, 15, 40$. We used $\sigma = 5, 15$ for CS-MRI and single photon imaging, and $\sigma = 40$ for Poisson denoising. The regularized denoisers are trained to have Lipschitz constant (no more than) 1. The training data consists of images from the BSD500 dataset, divided into $40 \times 40$ patches. The CNN weights were initialized in the same way as \cite{zhang2017beyond}. We train all networks using the ADAM optimizer for 50 epochs, with a mini-batch size of 128. The learning rate was $10^{-3}$ in the first 25 epochs, then decreased to $10^{-4}$. 
On an Nvidia GTX 1080 Ti, DnCNN took 4.08 hours and realSN-DnCNN took 5.17 hours to train, so the added cost of realSN is mild.

\section{Poisson denoising: validating the theory}
\label{s:poisson_experiments}
Consider the Poisson denoising problem,
where given a true image $x_\mathrm{true}\in \reals^d$, we observe independent Poisson random variables $y_i \sim \text{Poisson}((x_\mathrm{true})_i)$, so $y_i\in \mathbb{N}$, for $i=1,\dots,d$. 
For details and motivation for this problem setup, see \cite{rond2016}.

For the objective function $f(x)$, we use the negative log-likelihood  given by $f(x) = \sum^d_{i=1}\ell(x_i;y_i)$,
where
\[
\ell(x;y)=\left\{
\begin{array}{ll}
-y\log (x)+x&\text{ for }y>0,\,x>0\\
0&\text{ for }y=0,\,x\ge 0\\
\infty&\text{ otherwise.}
\end{array}
\right.
\]
We can compute $\Prox_{\alpha f}$ elementwise with
\vspace{-0.25em}
\[
\Prox_{\alpha f}(x)=
(1/2)\left(x-\alpha +\sqrt{(x-\alpha)^2+4\alpha y}\right).
\]
\vspace{-0.25em}
The gradient of $f$ is given by $\partial f/\partial x_i=-y_i/x_i+1$ for $x_i>0$ for $i=1,\dots,d$.
We set $\partial f/\partial x_i=0$ when $x_i=0$, although, strictly speaking, $\partial f/\partial x_i$ is undefined when $y_i>0$ and $x_i=0$.
This does not seem to cause any problems in the experiments.
Since we force the denoisers to output nonnegative pixel values,
PNP-FBS never needs to evaluate $\partial f/\partial x_i$ for negative $x_i$.

%For $H_\sigma$, we choose BM3D \cite{dabov2007}, SimpleCNN, and DnCNN \cite{zhang2017beyond}.
For $H_\sigma$, we choose BM3D, SimpleCNN with and without realSN, and DnCNN with and without realSN.
%We further enforce realSN as the regularization to training SimpleCNN and DnCNN, denoted as SimpleCNN-realSN and DnCNN-realSN, respectively.
Note that these denoisers are designed or trained for the purpose of \textbf{Gaussian denoising}, and here we integrate them into the PnP frameworks for Poisson denoising.
We scale the image so that the peak value of the image, the maximum mean of the Poisson random variables, is $1$.
The $y$-variable was initialized to the noisy image for PnP-FBS and PnP-ADMM, and the $u$-variable was initialized to $0$ for PnP-ADMM.
We use the test set of 13 images in \cite{chan2017}.

\vspace{-0.5em}
\paragraph{Convergence.}
We first examine which denoisers satisfy Assumption~\eqref{assumption:H} with small $\varepsilon$.
In Figure~\ref{fig:histogram}, we run PnP iterations of Poisson denoising on a single image (flag of \cite{rond2016}) with different models, calculate $\|(I-H_\sigma)(x)-(I-H_\sigma)(y)\|/\|x-y\|$ between the iterates and the limit, and plot the histogram.
% In Figure~\ref{fig:histogram}, we plot histograms of $\|(I-H_\sigma)(x)-(I-H_\sigma)(y)\|/\|x-y\|$ for a set of inputs $x,y\in\reals^d$.
The maximum value of a histogram, marked by a vertical red bar, lower-bounds the $\varepsilon$ of Assumption~\eqref{assumption:H}.
Remember that Corollary~\ref{cor:contraction} requires $\varepsilon<1$ to ensure convergence of PnP-ADMM.
Figure~\ref{fig:hist_bm3d} proves that BM3D violates this assumption.
%Therefore, any convergence of BM3D is not backed by theory.
%Figures~\ref{fig:hist_simplecnn} and \ref{fig:hist_simplecnn2} and
Figures \ref{fig:hist_simplecnn} and \ref{fig:hist_simplecnn2} and Figures~\ref{fig:hist_dncnn} and \ref{fig:hist_dncnn2} respectively illustrate that RealSN indeed improves (reduces) $\varepsilon$ for SimpleCNN and DnCNN.

%tell us that the real spectral normalization technique indeed improves (reduces) $\varepsilon$ for DnCNN, but not simpleCNN.

Figure~\ref{fig:contractions} experimentally validates Theorems~\ref{thm:fbs-contraction} and \ref{thm:contraction}, by examining the average (geometric mean) contraction factor (defined in Section \ref{s:ct}) of PnP-FBS and ADMM\footnote{We compute the contraction factor of the equivalent PnP-DRS.} iterations over a range of step sizes.
%See Figure~\ref{fig:contractions}.
Figure~\ref{fig:contractions}  qualitatively shows that PnP-ADMM exhibits more stable convergence than PnP-FBS.
Theorem~\ref{thm:fbs-contraction} ensures PnP-FBS is a contraction when $\alpha$ is within an interval and Theorem~\ref{thm:contraction} ensures PnP-ADMM is a contraction when $\alpha$ is large enough.
We roughly observe this behavior for the denoisers trained with RealSN.

\setlength\tabcolsep{3pt}
\begin{table}[h]
\centering
\small
\vspace{-1em}
\caption{Average PSNR performance (in dB) on Poisson denoising (peak $= 1$) on the testing set in \cite{chan2017}}
\label{tab:poisson_best}
\begin{tabular}{|c|c|c|c|}
\hline
% \multicolumn{6}{c}{PnP-DRS, $\alpha = 0.1$}\\ \hline
& BM3D    &   RealSN-DnCNN   & RealSN-SimpleCNN \\ 
%& BM3D    &   RSN-DnCNN   & RSN-SimpleCNN \\ 
%& & DnCNN & SimpleCNN \\
\hline\hline
%Iteration 50  & 23.3393 &  23.4580  &  23.4709  &  14.4165  & 18.7877\\ \hline
%Iteration 100 & 23.3227 &  23.4785  &  23.4834  &  14.2642  & 18.7880\\ \hline
%Best Overall            & 23.4617 &  23.6368  &  23.5873  &  16.7160  & 18.7890 \\
PNP-ADMM           & 23.4617   &  \textbf{23.5873}    & 18.7890 \\ \hline
%\multicolumn{6}{c}{PnP-FBS, $\alpha = 0.0125$}\\ \hline
%Average PSNR         & BM3D    &  DnCNN    &  RealSN-DnCNN   & SimpleCNN  & RealSN-SimpleCNN \\ \hline\hline
%Iteration 50  & 13.9490 &  18.5296  &  20.9351  &  17.5798  & 18.4302\\ \hline
%Iteration 100 & 18.2140 &  18.5296  &  20.7346  &  16.7609  & 15.1735\\ \hline
%Best Overall  & 18.5835 &  22.6624  &  22.2154  &  21.9649  & 22.7280 \\ \hline\hline
PNP-FBS  & 18.5835   &  22.2154  & \textbf{22.7280} \\ \hline
\end{tabular}
%\vspace{-0.5em}
\end{table}
\vspace{-1em}
\paragraph{Empirical performance.}
Our theory only concerns convergence and says nothing about the recovery performance of the output the methods converge to.
We empirically verify that the PnP methods with RealSN, for which we analyzed convergence, yield competitive denoising results.

We fix $\alpha=0.1$ for all denoisers in PNP-ADMM, and $\alpha=0.0125$ in PNP-FBS. For deep learning-based denoisers, we choose $\sigma=40/255$.
For BM3D, we choose $\sigma = \sqrt{\gamma \alpha}$ as suggested in \cite{rond2016} and use $\gamma=1$.

Table \ref{tab:poisson_best} compares the PnP methods with BM3D, RealSN-DnCNN, and RealSN-SimpleCNN plugged in. 
%The PnP schemes evidently improves all three Gaussian denoisers for this task. 
In both PnP methods, one of the two denoisers using RealSN, for which we have theory, outperforms BM3D.
It is interesting to obverse that the PnP performance does not necessarily hinge on the strength of the plugged in denoiser and that different PnP methods favor different denoisers.
For example, RealSN-SimpleCNN surpasses the much more sophisticated RealSN-DnCNN under PnP-FBS. However, RealSN-DnCNN leads to better, and overall best, denoising performance when plugged into PnP-ADMM.

\begin{figure}[]
\centering
%\hspace{-5mm}
\subfigure[PnP-FBS]{
      \includegraphics[width=0.45\textwidth]{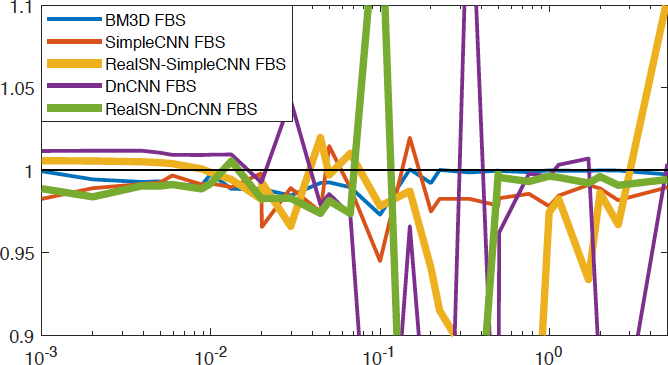}\label{doirjg}}\\
            %\vspace{-1.em}
      %\hspace{-10mm}
\subfigure[PNP-ADMM]{
      \includegraphics[width=0.45\textwidth]{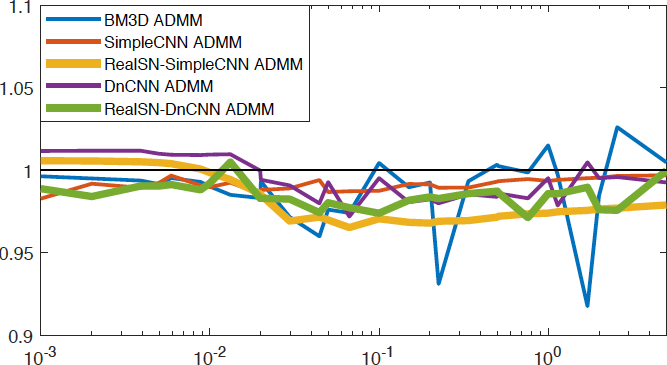}\label{sidfjasoi}}
      \vspace{-0.5em}
      \caption{Average contraction factor of 500 iterations for the Poisson denoising experiment.
      The x-axis represents the value of $\alpha$ and y-axis represents the contraction factor. Although lower means faster convergence,  a smoother curve means the method is easier to tune and has more stable convergence.
     }
     \vspace{-0.5em}
      \label{fig:contractions}
\end{figure}

\section{More applications}
\label{s:more_apple}
We now apply PnP on two imaging problems% where the objective functions $f$ are not strongly convex.
%Although without convergence guarantee here,
and show that RealSN improves the reconstruction of PnP.%
%Code is available at \url{https://github.com/xxxxxxxx/xxxx}
\footnote{Code for our experiments in 
Sections~\ref{s:poisson_experiments} and \ref{s:more_apple}
is available at \url{https://github.com/uclaopt/Provable_Plug_and_Play/}}

\newpage

\vspace{-0.5em}
\paragraph{Single photon imaging.}
Consider single photon imaging with quanta image sensors (QIS) \cite{fossum2011quanta,chan2014efficient,elgendy2016image}
with the model
\[z = \mathbf{1}(y \geq 1), \quad y \sim \text{Poisson}(\alpha_{sg} G x_\mathrm{true})\]
where $x_\mathrm{true}\in\reals^{d}$ is the underlying image, $G:\reals^{d}\to\reals^{dK}$ duplicates each pixel to $K$ pixels, $\alpha_{sg}\in \reals$ is sensor gain, $K$ is the oversampling rate, $z\in \{0,1\}^{dK}$ is the observed binary photons.  We want to recover $x_\mathrm{true}$ from $z$. %\jialin{Here we take $G$ as an operator that duplicating each pixel by $K$ times. With this operator $G$,}
The likelihood function is
\[f(x) = \sum_{j=1}^n -K^0_j \log(e^{-\alpha_{sg} x_j/K}) - K^1_j \log(1-e^{-\alpha_{sg} x_j/K}),\]
where $K^1_j = \sum_{i=1}^K z_{(j-1)K+i}$ is the number of ones in the $j$-th unit pixel, $K^0_j = \sum_{i=1}^K 1-z_{(j-1)K+i}$ is the number of zeros in the $j$-th unit pixel. The gradient of $f(x)$ is given by $\partial f / \partial x_j = (\alpha_{sg}/K) (K^0_j - K^1_j / (e^{\alpha_{sg} x_j/K} -1))$ and the proximal operator of $f$ is given in \cite{chan2014efficient}. 
%{\color{red} accurate? cite?}
%\ernest{What is the role of $G$?}

We compare PnP-ADMM and PnP-FBS respectively with the denoisers BM3D, RealSN-DnCNN, and RealSN-SimpleCNN. We take $\alpha_{sg} = K = 8$.
The $y$-variable was initialized to $K^1$ for PnP-FBS and PnP-ADMM, and the $u$-variable was initialized to $0$ for PnP-ADMM.
All deep denoisers used in this experiment were trained with fixed noise level $\sigma= 15$. We report the PSNRs achieved at the 50th iteration, the 100th iteration, and the best PSNR values achieved within the first 100 iterations. 

Table \ref{tab:SPI} reports the average PSNR results on the 13 images used in \cite{chan2017}.
Experiments indicate that PnP-ADMM methods constantly yields higher PNSR than the PnP-FBS counterparts using the same denoiser. The best overall PSNR is achieved using PnP-ADMM with RealSN-DnCNN, which shows nearly 1dB improvement over the result obtained with BM3D. We also observe that deep denoisers with RealSN make PnP converges more stably.

\setlength\tabcolsep{6pt}

\begin{table*}
\centering
\caption{Average PSNR (in dB) of single photon imaging task on the test set of \cite{chan2017} }
\label{tab:SPI}
\begin{tabular}{c|c|c|c|}
\hline
\multicolumn{4}{c|}{PnP-ADMM, $\alpha = 0.01$}\\ \hline
Average PSNR  & BM3D    &  RealSN-  & RealSN- \\
 &    &  DnCNN  & SimpleCNN \\
\hline\hline
Iteration 50  & 30.0034   &  31.0032        & 29.2154\\ \hline
Iteration 100 & 30.0014   &  31.0032        & 29.2151\\ \hline
Best Overall  & 30.0474   &  \textbf{31.0431}        & 29.2155 \\ \hline
\end{tabular}
\begin{tabular}{|c|c|c|c}
\hline
\multicolumn{4}{|c}{PnP-FBS, $\alpha = 0.005$}\\ \hline
Average PSNR  & BM3D    &  RealSN-  & RealSN- \\
 &    &  DnCNN  & SimpleCNN \\ \hline\hline
Iteration 50  & 28.7933   &  27.9617       & 29.0062\\ \hline
Iteration 100 & 29.0510   &  27.9887       & 29.0517\\ \hline
Best Overall  & \textbf{29.5327}   &  28.4065       & 29.3563 \\ \hline
\end{tabular}
\end{table*}

% \begin{figure*}
% \centering
% \begin{tabular}{ccc}
% \hspace{-5mm}
% \subfigure[BM3D]{
%       \includegraphics[width=0.35\textwidth]{img/srg_bm3d_admm_bm3d_lam10.eps}\label{fig:admm_kend_bm3d}}
% &
% \hspace{-20mm}
% \subfigure[DnCNN]{
%       \includegraphics[width=0.35\textwidth]{img/srg_dncnn_admm_cnn_lam10.eps}\label{fig:admm_kend_nlm}}
%   &
% \hspace{-20mm}
% \subfigure[Real SN]{
%       \includegraphics[width=0.35\textwidth]{img/srg_sn11111_admm_lam10.png}\label{fig:admm_kend_nlm}}
% \end{tabular}
% \caption{SRG of denoisers: check the assumption A}
% \label{fig:case3}
% \end{figure*}

% \begin{figure*}
% \centering
% \begin{tabular}{ccc}
% \hspace{-5mm}
% \subfigure[BM3D]{
%       \includegraphics[width=0.35\textwidth]{img/srg_admm_bm3d_lam10.eps}\label{fig:admm_kend_bm3d}}
% &
% \hspace{-20mm}
% \subfigure[DnCNN]{
%       \includegraphics[width=0.35\textwidth]{img/srg_admm_cnn_lam10.eps}\label{fig:admm_kend_nlm}}
%   &
% \hspace{-20mm}
% \subfigure[Real SN]{
%       \includegraphics[width=0.35\textwidth]{img/srg_admm_sn11111_lam10.png}\label{fig:admm_kend_nlm}}
% \end{tabular}
% \caption{SRG of PnP-ADMM operator $T$: effect of Lipschitz constant $\varepsilon$}
% \label{fig:case3}
% \end{figure*}

% \begin{figure*}
% \centering
% \begin{tabular}{cc}
% \hspace{-5mm}
% \subfigure[PnP-ADMM]{
%       \includegraphics[width=0.45\textwidth]{img/drs.eps}\label{fig:admm_kend_bm3d}}
% &
% \hspace{-10mm}
% \subfigure[PnP-FBS]{
%       \includegraphics[width=0.45\textwidth]{img/fbs.eps}\label{fig:admm_kend_nlm}}
% \end{tabular}
% \caption{Effect of $H_\sigma$ on the convergence}
% \label{fig:case3}
% \end{figure*}

\begin{table*}[h]
\centering
\caption{CS-MRI results (30\% sample with additive Gaussian noise $\sigma_e=15$) in PSNR (dB). %All deep denoisers are trained on natural images.
%but they work well in the PnP framework on medical imaging task\footnotemark. The ``spectral normalization" (RealSN) technique helps improve the reconstruction performance of PnP-ADMM and PnP-FBS.
} 
\label{tab:ct}
\begin{tabular}{c|c|c|c|c|c|c|c}
\hline
  \multicolumn{2}{c}{Sampling approach}  & \multicolumn{2}{|c}{Random} & \multicolumn{2}{|c}{Radial} & \multicolumn{2}{|c}{Cartesian}    \\ \hline
  \multicolumn{2}{c|}{Image} & Brain &  Bust  & Brain & Bust   & Brain & Bust \\\hline\hline
\multicolumn{2}{c|}{Zero-filling}  &    9.58      &     7.00     &       9.29    &             6.19     &   8.65                      &   6.01    \\ \hline
\multicolumn{2}{c|}{TV \cite{lustig2005application}}  &  16.92  &  15.31  &  15.61  &  14.22  &  12.77  &  11.72  \\ \hline
\multicolumn{2}{c|}{RecRF \cite{yang2010fast}} &     16.98    &     15.37         &    16.04   &              14.65       & 12.78                   &    11.75      \\ \hline
\multicolumn{2}{c|}{BM3D-MRI \cite{eksioglu2016decoupled}}    &      17.31    &      13.90&         16.95   &     13.72   &    14.43                          &  12.35    \\ \hline\hline

\multirow{5}{*}{PnP-FBS} & BM3D   &     19.09     &     16.36     &        18.10   &            15.67  &    14.37                    &   12.99      \\ \cline{2-8}
& DnCNN  &      19.59    &       16.49        &      18.92   &            15.99     &   14.76                  &    14.09     \\ \cline{2-8}
& RealSN-DnCNN & \textbf{19.82}            & \textbf{16.60} &  \textbf{18.96}         & \textbf{16.09}  &     \textbf{14.82}        &   \textbf{14.25}                     \\ \cline{2-8}
& SimpleCNN        & 15.58 & 12.19 & 15.06 & 12.02 & 12.78 & 10.80 \\ \cline{2-8}
& RealSN-SimpleCNN & 17.65 & 14.98 & 16.52 & 14.26 & 13.02 & 11.49 \\ \hline\hline

\multirow{5}{*}{PnP-ADMM} & BM3D & 19.61 &  \textbf{17.23}  &  18.94  &  \textbf{16.70}  &  14.91                      &   13.98     \\ \cline{2-8}
& DnCNN  &         19.86 &      17.05        &       19.00   &         16.64     &   14.86                    &   14.14        \\ \cline{2-8}
& RealSN-DnCNN &       \textbf{19.91}      &      {17.09}      &       \textbf{19.08}   &       {16.68}      &   \textbf{15.11}                &    \textbf{14.16}       \\ \cline{2-8}
& SimpleCNN        & 16.68 & 12.56  &  16.83 &  13.47  &    13.03    &  11.17  \\ \cline{2-8}
& RealSN-SimpleCNN & 17.77 & 14.89  &  17.00 &  14.47  &    12.73    &  11.88  \\ \hline
\end{tabular}
\end{table*}

\vspace{-0.5em}
\paragraph{Compressed sensing MRI.} 
Magnetic resonance imaging (MRI) is a widely-used imaging technique with
a slow data acquisition. Compressed sensing MRI (CS-MRI) accelerates MRI by acquiring less data through downsampling.
PnP is useful in medical imaging as we do not have a large amount of data for end-to-end training: we train the denoiser $H_\sigma$ on natural images, and then ``plug'' it into the PnP framework to be applied to medical images. CS-MRI is described mathematically as
\[y = \mathcal{F}_p x_{\mathrm{true}} + \varepsilon_e,
\]
where $x_{\mathrm{true}} \in \mathbb{C}^d$ is the underlying image, $\mathcal{F}_p:\mathbb{C}^d\to\mathbb{C}^k$ is the linear measurement model, $y \in \mathbb{C}^k$ is the measured data, and $\varepsilon_e\sim N(0,\sigma_eI_k)$ is measurement noise. We want to recover $x_\mathrm{true}$ from $y$. The objective function is
\[f(x) = (1/2) \|y - \mathcal{F}_p x\|^2.\]The gradient of $f(x)$ is given in \cite{liu2016projected} and the proximal operator of $f(x)$ is given in \cite{eksioglu2016decoupled}. 
We use BM3D, SimpleCNN and DnCNN, and their variants by RealSN for the PnP denoiser $H_\sigma$.
%{\color{red} accurate? cite?}

We take $\mathcal{F}_p$ as the Fourier k-domain subsampling (partial Fourier operator). We tested random, radial, and Cartesian sampling \cite{eksioglu2016decoupled} with a sampling rate of $30\%$. The noise level $\sigma_e$ is taken as $15/255$. 

We compare PnP frameworks with zero-filling, total-variation (TV) \cite{lustig2005application}, RecRF \cite{yang2010fast}, and BM3D-MRI \cite{eksioglu2016decoupled} \footnote{Some recent deep-learning based methods \cite{sun2016deep,kulkarni2016reconnet,metzler2017learned,zhang2017ista} are not compared here because we assume we do not have enough medical images for training.}. The parameters are taken as follows. For TV, the regularization parameter $\lambda$ is taken as the best one from $\{a \times 10^b, a \in \{1,2,5\}, b\in \{-5,-4,-3,-2,-1,0,1\}\}$. For RecRF, the two parameters $\lambda,\mu$ are both taken from the above sets and the best results are reported. For BM3D-MRI, we set the ``final noise level (the noise level in the last iteration)" as $2\sigma_e$, which is suggested in their MATLAB library. For PnP methods with $H_\sigma$ as BM3D, we set $\sigma = 2\sigma_e$, take $\alpha \in \{0.1,0.2,0.5,1,2,5\}$ and report the best results. For PNP-ADMM with $H_\sigma$ as deep denoisers, we take $\sigma = \sigma_e = 15/255$ and $\alpha = 2.0$ uniformly for all the cases. For PNP-FBS with $H_\sigma$ as deep denoisers, we take $\sigma = \sigma_e /3 = 5/255$ and $\alpha = 0.4$ uniformly. All deep denoisers are trained on BSD500 \cite{MartinFTM01}, a natural image data set; no medical image is used in training.
The $y$-variable was initialized to the zero-filled solution for PnP-FBS and PnP-ADMM, and the $u$-variable was initialized to $0$ for PnP-ADMM.
Table \ref{tab:ct} reports our results on CS-MRI, from which we can confirm the effectiveness of PnP frameworks. Moreover, using RealSN-DnCNN seems to the clear winner over all. We also observe that PnP-ADMM generally outperforms PnP-FBS when using the same denoiser, which supports Theorems \ref{thm:fbs-contraction} and \ref{thm:contraction}.

 \section{Conclusion}
In this work, we analyzed the convergence of PnP-FBS and PnP-ADMM under a Lipschitz assumption on the denoiser.
We then presented real spectral normalization a technique to enforce the proposed Lipschitz condition in training deep learning-based denoisers.
Finally, we validate the theory with experiments.
%  Code for experiments 
%  is available at \url{https://github.com/uclaopt/Provable_Plug_and_Play/}

\section*{Acknowledgements}
We thank Pontus Giselsson for the discussion on negatively averaged operators and Stanley Chan for the discussion on the difficulties in establishing convergence of PnP methods.
This work was partially supported by National Key R\&D Program of China 2017YFB02029, AFOSR MURI FA9550-18-1-0502, NSF DMS-1720237, ONR N0001417121, and NSF RI-1755701.

\bibliography{example_paper}
\bibliographystyle{icml2019}

\onecolumn
\newpage

\section{Preliminaries}
%\subsection{Convex analysis}
For any $x,y\in \reals^d$, write $\langle x,y\rangle =x^Ty$ for the inner product.
We say a function $f:\reals^d\rightarrow\reals\cup\{\infty\}$ is convex if
\[
f(\theta x+(1-\theta)y)\le \theta f(x)+(1-\theta)f(y)
\]
for any $x,y\in \reals^d$ and $\theta\in[0,1]$.
A convex function is closed if it is lower semi-continuous and proper if it is finite somwhere.
We say $f$ is $\mu$-strongly convex for $\mu>0$ if
$f(x)-(\mu/2)\|x\|^2$
is a convex function.
% We say $\nabla f$ is $L$-Lipschitz continuous if
% $f$ is differentiable and if
% \[
% \|\nabla f(x)-\nabla f(y)\|^2\le L^2\|x-y\|^2
% \]
% for all $x,y\in \reals^d$.
% The inequality
% \[
% \mu\le L
% \]
% always holds.
 Given a convex function $f:\reals^d\rightarrow\reals\cup\{\infty\}$ and $\alpha>0$, define its proximal operator 
 $\Prox_f:\reals^d\rightarrow\reals^d$ as
 \[
 \Prox_{\alpha f}(z)=\argmin_{x\in \reals^d}
 \left\{ \alpha f(x)+(1/2)\|x-z\|^2 \right\}.
 \]
When $f$ is convex, closed, and proper, the $\argmin$ uniquely exists, and therefore $\Prox_f$ is well-defined.
An mapping $T:\reals^d\rightarrow \reals^d$ is $L$-Lipschitz if
\[
\|T(x)-T(y)\|\le L\|x-y\|
\]
for all $x,y,\in\reals^d$. 
If $T$ is $L$-Lipschitz with $L\le 1$, we say $T$ is nonexpansive.
If $T$ is $L$-Lipschitz with $L< 1$, we say $T$ is a contraction.
% and is a contraction if 
% \[
% \|T(x)-T(y)\|\le \gamma\|x-y\|
% \]
% for some $\gamma<1$,
A mapping $T:\reals^d\rightarrow \reals^d$ is $\theta$-averaged for $\theta\in (0,1)$,
if it is nonexpansive and if
\[
T=\theta R+(1-\theta)I,
\]
where $R:\reals^d\rightarrow\reals^d$ is another nonexpansive mapping.

\begin{lemma}
[Proposition~4.35 of \cite{BauschkeCombettes2017_convex}]
\label{lem:averagedness}
$T:\reals^d\rightarrow\reals^d$ is $\theta$-averaged if and only if
\[
\|T(x)-T(y)\| ^2
+(1-2\theta)\|x-y\| ^2\le
2(1-\theta)
\langle
T(x)-T(y),x-y\rangle
\]
for all $x,y\in \reals^d$.
\end{lemma}

% \cite[Theorem~3(b)]{ogura2002},
% \cite[Proposition~2.4]{combettes2015},
% \cite[Proposition~4.44]{BauschkeCombettes2017_convex}
\begin{lemma}[\cite{ogura2002,combettes2015}]
\label{lem:avg-comp}
Assume $T_1:\reals^d\rightarrow\reals^d$
and
$T_2:\reals^d\rightarrow\reals^d$
are
$\theta_1$ and $\theta_2$-averaged, respectively.
Then $T_1 T_2$ is 
$\frac{\theta_1+\theta_2- 2\theta_1\theta_2}{1-\theta_1\theta_2}$-averaged.
\end{lemma}

\begin{lemma}
\label{lem:neg-avg}
Let $T:\reals^d\rightarrow\reals^d$.
$-T$ is $\theta$-averaged if and only if
$T\circ(-I)$ is $\theta$-averaged.
\end{lemma}
\begin{proof}
The lemma follows from the fact that 
\[
T\circ (-I)=\theta R+(1-\theta)I
\quad\Leftrightarrow\quad
-T =\theta (-R)\circ (-I)+(1-\theta)I
\]
for some nonexpansive $R$ and that nonexpansiveness of $R$ and implies nonexpansivenes of  $-R\circ (-I)$.
\end{proof}

\begin{lemma}[\cite{taylor2018}]
\label{lem:grad-cont}
Assume $f$ is $\mu$-strongly convex and $\nabla f$ is $L$-Lipschitz. Then for any $x,y\in \reals^d$,
we have
\[
\|(I-\alpha \nabla f)(x)-(I-\alpha \nabla f)(y)\|\le 
\max\{|1-\alpha\mu|,|1-\alpha L|\}
\|x-y\|.
\]
\end{lemma}

\begin{lemma}[Proposition~5.4 of \cite{giselsson2017}]
\label{lem:giselsson}
Assume $f$ is $\mu$-strongly convex, closed, and proper.
Then
\[
-(2\Prox_{\alpha f}-I)
\]
is $\frac{1}{1+\alpha\mu}$-averaged.
\end{lemma}

\paragraph{References.}
The notion of proximal operator and  its well-definedness were first presented in \cite{moreau1965}.
The notion of averaged mappings were first introduced in \cite{bailion1978}.
%Specifically, Lemma~\ref{lem:avg-comp} is presented as Theorem 3(b) in  \cite{ogura2002}, Proposition 2.4 in  \cite{combettes2015}, and Proposition 4.44 in  \cite{BauschkeCombettes2017_convex}.
 The idea of Lemma~\ref{lem:neg-avg} relates to ``negatively averaged'' operators from  \cite{giselsson2017}.
Lemma~\ref{lem:grad-cont} is proved in a weaker form as Theorem 3 of \cite{polyak1987} and in Section 5.1 of \cite{ryu2016}.
Lemma~\ref{lem:grad-cont} as stated is proved as Theorem 2.1 in \cite{taylor2018}.

\section{Proofs of main results}
\subsection{Equivalence of PNP-DRS and PNP-ADMM}
We show the standard steps that establish equivalence of PNP-DRS and PNP-ADMM.
%We now that PNP-DRS is equivalent to PNP-ADMM.
Starting from PNP-DRS, we substitute $z^k=x^k+u^k$ to get
\begin{align*}
x^{k+1/2}&=
\Prox_{\alpha f}(x^k+u^k)\\
x^{k+1}&=H_\sigma(x^{k+1/2}-(u^k+x^k-x^{k+1/2}))\\
u^{k+1}&=u^k+x^{k}-x^{k+1/2}.
\end{align*}
We reorder the iterations to get the correct dependency
\begin{align*}
x^{k+1/2}&=
\Prox_{\alpha f}(x^k+u^k)
\\
u^{k+1}&=u^k+x^{k}-x^{k+1/2}\\
x^{k+1}&=H_\sigma(x^{k+1/2}-u^{k+1}).
\end{align*}
We label $\tilde{y}^{k+1}=x^{k+1/2}$ and $\tilde{x}^{k+1}=x^{k}$
\begin{align*}
\tilde{x}^{k+1}&=H_\sigma(\tilde{y}^{k}-u^{k})\\
\tilde{y}^{k+1}&=
\Prox_{\alpha f}(\tilde{x}^{k+1}+u^k)\\
u^{k+1}&=u^k+\tilde{x}^{k+1}-\tilde{y}^{k+1},
\end{align*}
and we get PNP-ADMM.

\subsection{Convergence analysis}
\begin{lemma}
\label{lem:Hlemma}
$H_\sigma:\reals^d\rightarrow\reals^d$ satisfies Assumption \eqref{assumption:H} if and only if
\[
\frac{1}{1+\varepsilon}H_\sigma
\]
is nonexpansive and $\frac{\varepsilon}{1+\varepsilon}$-averaged.
%for $\varepsilon\in[0,1)$.
\end{lemma}
\begin{proof}
Define $\theta = \frac{\varepsilon}{1+\varepsilon}$, which means $\varepsilon=\frac{\theta}{1-\theta}$. 
Clearly, $\theta\in [0,1)$.
Define $G=\frac{1}{1+\varepsilon}H_\sigma$, which means
$H_\sigma=\frac{1}{1+\theta}G$.
Then
\begin{align*}
&\underbrace{\|(H_\sigma-I)(x)-(H_\sigma-I)(y)\|^2-\frac{\theta^2}{(1-\theta)^2} \|x-y\|^2}_{\text{(TERM A)}}\\
&\quad=\frac{1}{(1-\theta)^2}\|G(x)-G(y)\|^2+\left(1-\frac{\theta^2}{(1-\theta)^2}\right)\|x-y\|^2-\frac{2}{1-\theta}\langle G(x)-G(y),x-y\rangle\\
&\quad=
\frac{1}{(1-\theta)^2}
\bigg(
\underbrace{\|G(x)-G(y)\| ^2
+(1-2\theta)\|x-y\| ^2
-2(1-\theta)
\langle
G(x)-G(y),x-y\rangle}_{\text{(TERM B)}}
\bigg).
\end{align*}
Remember that Assumption \eqref{assumption:H} corresponds to
$\text{(TERM A)}\le 0$ for all $x,y\in \reals^d$.
This is equivalent to $\text{(TERM B)}\le 0$ for all $x,y\in \reals^d$, which corresponds to $G$ being $\theta$-averaged by Lemma~\ref{lem:averagedness}.
\end{proof}

\begin{lemma}
\label{lem:Hlemma2}
$H_\sigma:\reals^d\rightarrow\reals^d$ satisfies Assumption \eqref{assumption:H} if and only if
\[
\frac{1}{1+2\varepsilon}(2H_\sigma-I)
\]
is nonexpansive and $\frac{2\varepsilon}{1+2\varepsilon}$-averaged.
\end{lemma}
\begin{proof}
Define $\theta = \frac{2\varepsilon}{1+2\varepsilon}$, which means $\varepsilon=\frac{\theta}{2(1-\theta)}$.
Clearly, $\theta\in [0,1)$.
Define $G=\frac{1}{1+2\varepsilon}(2H_\sigma-I)$, which means $H_\sigma=\frac{1}{2(1-\theta)}G+\frac{1}{2}I$. 
Then
\begin{align*}
&\underbrace{\|(H_\sigma-I)(x)-(H_\sigma-I)(y)\|^2-\frac{\theta^2}{4(1-\theta)^2} \|x-y\|^2}_{\text{(TERM A)}}\\
&\quad=\frac{1}{4(1-\theta)^2}\|G(x)-G(y)\|^2
+
\left(
\frac{1}{4}
-\frac{\theta^2}{4(1-\theta)^2}
\right)
\|x-y\|^2-\frac{1}{2(1-\theta)}\langle G(x)-G(y),x-y\rangle\\
&\quad=
\frac{1}{4(1-\theta)^2}
\bigg(
\underbrace{
\|G(x)-G(y)\| ^2
+(1-2\theta)\|x-y\| ^2
-2(1-\theta)
\langle
G(x)-G(y),x-y\rangle}_{\text{(TERM B)}}
\bigg).
\end{align*}
Remember that Assumption \eqref{assumption:H} corresponds to
$\text{(TERM A)}\le 0$ for all $x,y\in \reals^d$.
This is equivalent to $\text{(TERM B)}\le 0$ for all $x,y\in \reals^d$, which corresponds to $G$ being $\theta$-averaged by Lemma~\ref{lem:averagedness}.
\end{proof}

\begin{proof}[\textbf{Proof of Theorem~\ref{thm:fbs-contraction}}]
In general, if operators $T_1$ and $T_2$ are $L_1$ and $L_2$-Lipschitz, then the composition $T_1T_2$ is $(L_1L_2)$-Lipschitz.
By Lemma~\ref{lem:grad-cont}, $I-\alpha \nabla f$ is $\max\{|1-\alpha\mu|,|1-\alpha L|\}$-Lipschitz.
By Lemma~\ref{lem:Hlemma}, $H_\sigma$ is $(1+\varepsilon)$-Lipschitz.
The first part of the theorem following from composing the Lipschitz constants.
The restrictions on $\alpha$ and $\varepsilon$ follow from basic algebra.
\end{proof}

\begin{proof}[\textbf{Proof of Theorem~\ref{thm:contraction}}]
By Lemma~\ref{lem:giselsson}, 
\[
-(2\Prox_{\alpha f}-I)
\]is $\frac{1}{1+\alpha\mu}$-averaged, and this implies 
\[
(2\Prox_{\alpha f}-I)\circ (-I)
\]
is also $\frac{1}{1+\alpha\mu}$-averaged, by Lemma~\ref{lem:neg-avg}.
By Lemma~\ref{lem:Hlemma2}, 
\[
\frac{1}{1+2\varepsilon}(2H_\sigma-I)
\]is $\frac{2\varepsilon}{1+2\varepsilon}$-averaged.
Therefore,
\[
\frac{1}{1+2\varepsilon}(2H_\sigma-I)(2\Prox_{\alpha f}-I)\circ (-I)
\]
is $\frac{1+2\varepsilon\alpha\mu}{1+\alpha\mu+2\varepsilon\alpha\mu}$\nobreakdash-averaged by Lemma~\ref{lem:avg-comp}, and this implies
\[
-\frac{1}{1+2\varepsilon}(2H_\sigma-I)(2\Prox_{\alpha f}-I)
\]
is also $\frac{1+2\varepsilon\alpha\mu}{1+\alpha\mu+2\varepsilon\alpha\mu}$\nobreakdash-averaged, by Lemma~\ref{lem:neg-avg}.
% \[
% \frac{1+2\varepsilon\alpha\mu}{1+\alpha\mu+2\varepsilon\alpha\mu}
% \]
%$(1+2\varepsilon\alpha\mu)/(1+\alpha\mu+2\varepsilon\alpha\mu)$\nobreakdash-averaged.

Using the definition of averagedness, we can write
\[
(2H_\sigma-I)(2\Prox_{\alpha f}-I)
=-(1+2\varepsilon)
\left(
\frac{\alpha\mu}{1+\alpha\mu+2\varepsilon\alpha\mu}
I
+
\frac{1+2\varepsilon\alpha\mu}{1+\alpha\mu+2\varepsilon\alpha\mu}R
\right)
\]
where $R$ is a nonexpansive operator. Plugging this into the PNP-DRS operator, we get
\begin{align}
T&=\frac{1}{2}I-\frac{1}{2}(1+2\varepsilon)\left(
\frac{\alpha\mu}{1+\alpha\mu+2\varepsilon\alpha\mu}
I
+
\frac{1+2\varepsilon\alpha\mu}{1+\alpha\mu+2\varepsilon\alpha\mu}R
\right)\nonumber\\
&=
\underbrace{\frac{1}{2(1+\alpha\mu+2\varepsilon\alpha\mu)}}_{=A}I
-\underbrace{\frac{(1+2\varepsilon\alpha\mu)(1+2\varepsilon)}{2(1+\alpha\mu+2\varepsilon\alpha\mu)}}_{=B}R,
\label{eq:drs_expression}
\end{align}
where define the coefficients $A$ and $B$ for simplicity. Clearly, $A>0$ and $B>0$.
Then we have
\begin{align*}
\|Tx-Ty\|^2&= A^2\|x-y\|^2+B^2\|R(x)-R(y)\|^2-2\langle A(x-y),B(R(x)-R(y))\rangle\\
&\le A^2\left(
1+\frac{1}{\delta}
\right)
\|x-y\|^2+B^2\left(1+\delta\right)\|R(x)-R(y)\|^2\\
&\le \left(A^2\left(
1+\frac{1}{\delta}
\right)
+B^2\left(1+\delta\right)\right)\|x-y\|^2
\end{align*}
for any $\delta>0$.
The first line follows from plugging in \eqref{eq:drs_expression}. 
The second line follows from applying Young's inequality to the inner product.
The third line follows from nonexpansiveness of $R$.

Finally, we optimize the bound. It is a matter of simple calculus to see
\[
\min_{\delta>0}\left\{A^2\left(1+\frac{1}{\delta}\right)+B^2\left(1+\delta\right)\right\}=(A+B)^2.
\]
%with minimizer $\delta=\frac{A}{B}$.
Plugging this in, we get
\begin{align*}
\|Tx-Ty\|^2&\le (A+B)^2\|x-y\|^2=
\left(\frac{1+\varepsilon+\varepsilon\alpha\mu+2 \varepsilon^2\alpha\mu}{1+\alpha\mu+2\varepsilon\alpha\mu}
\right)^2\|x-y\|^2,
\end{align*}
which is the first part of the theorem.

The restrictions on $\alpha$ and $\varepsilon$ follow from basic algebra.
\end{proof}

\begin{figure*}[h]
  \centering
  \includegraphics[width=\textwidth]{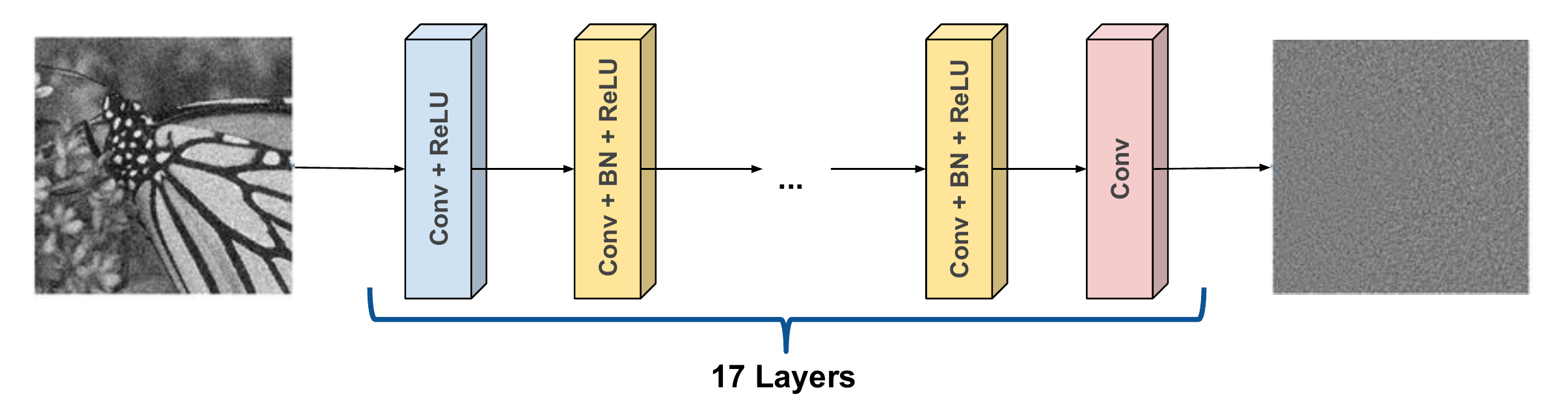}
  \caption{DnCNN Network Architecture}
%  \vspace{-2em}
    \label{model}
\end{figure*}

\begin{figure*}[h]
  \centering
  \includegraphics[width=\textwidth]{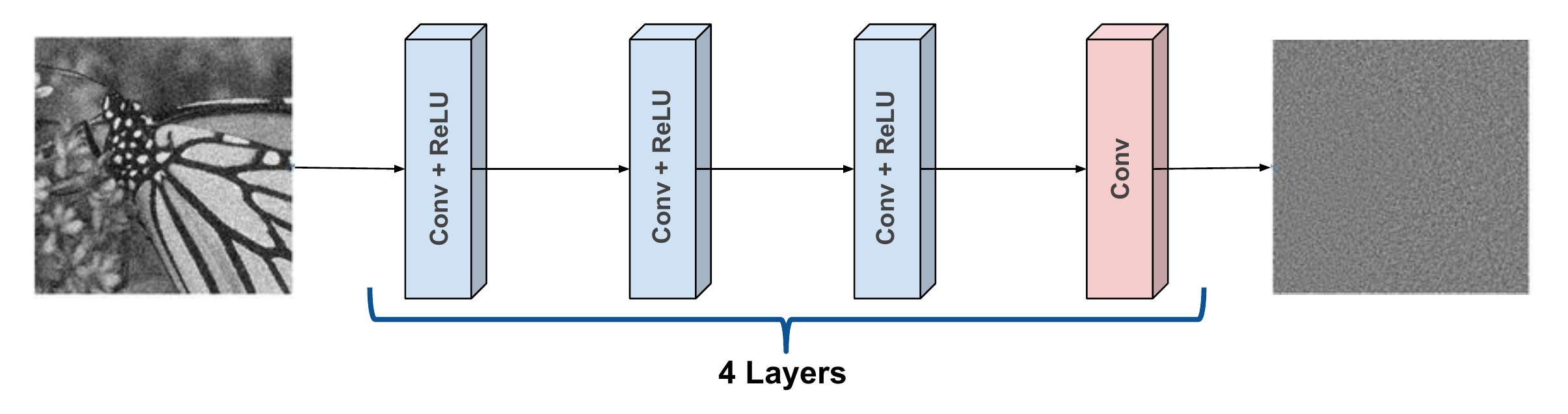}
  \caption{SimpleCNN Network Architecture}
%  \vspace{-2em}
    \label{model}
\end{figure*}

\end{document}